\documentclass[conference]{IEEEtran}

\usepackage{graphicx}
\usepackage{amssymb}
\usepackage{amsmath}
\usepackage{mathrsfs}
\usepackage{xspace}

\usepackage{color}
\usepackage{pgfplots}
\pgfplotsset{compat=newest}

\usepackage{times}
\usepackage{tikz}
\usetikzlibrary{arrows}
\usetikzlibrary{fit}
\usetikzlibrary{shapes.geometric}
\usetikzlibrary{trees}
\usepackage{verbatim}
\usepackage{graphicx}
\usepackage{amsthm}
\newtheorem{definition}{Definition}
\usepackage[onelanguage, ruled, vlined]{algorithm2e}

\newtheorem{example}{Example}
\newtheorem{theorem}{Theorem}

\begin{document}

\newcommand{\SyncBT}{SyncBT}
\newcommand{\nogoodm}{{\tt nogood}}
\newcommand{\okm}{{\tt ok?}}
\newcommand{\addlinkm}{{\tt addlink}}
\newcommand{\agentview}{{\tt agent view}}
\newcommand{\ABT}{ABT}
\newcommand{\StudentAlice}{Student~{\ensuremath{A_2}}}
\newcommand{\StudentBob}{Student~{\ensuremath{A_3}}}
\newcommand{\ProfessorP}{Professor~{\ensuremath{A_1}}}
\newcommand{\Alice}{{\ensuremath{A_2}}}
\newcommand{\Bob}{{\ensuremath{A_3}}}
\newcommand{\Professor}{{\ensuremath{A_1}}}
\newcommand{\risk}{\ensuremath{{\mbox{\,\it f}}utilityRisk}}
\newcommand{\ABTU}{ABTU}
\newcommand{\SyncBTU}{SyncBTU}
\newcommand{\estimatedCost}{$estimatedCost$}
\newcommand{\costNonTerminal}{\ensuremath{costNonTerminal}}
\newcommand{\costTerminal}{\ensuremath{costRound}}
\newcommand{\probabilityD}{\ensuremath{probD}}
\def\Aa{Professor}
\def\Aaa{{Professor\ensuremath{A_1}}}
\def\Aab{{Student\ensuremath{A_2}}}
\def\Aac{{Student\ensuremath{A_3}}}
\def\Ab{Student A2}
\def\Ac{Student A3}
\def\M{M}

\title{DisCSPs with Privacy Recast as Planning Problems for Utility-based Agents}

\author{
\IEEEauthorblockN{Julien Savaux, \\ Julien Vion, \\ Sylvain Piechowiak, 
\\ and Ren\'{e} Mandiau}
\IEEEauthorblockA{University of Valenciennes}
\and
\IEEEauthorblockN{Toshihiro Matsui}
\IEEEauthorblockA{Nagoya Institute of Technlogy}
\\[0mm]
\IEEEauthorblockN{Katsutoshi Hirayama}
\IEEEauthorblockA{Kobe University}
\and
\IEEEauthorblockN{Makoto Yokoo}
\IEEEauthorblockA{Kyushu University}
\\[0mm]
\IEEEauthorblockN{Shakre Elmane, Marius Silaghi}
\IEEEauthorblockA{Florida Institute of Technology}

}

\maketitle

\begin{abstract}
Privacy has traditionally been a major motivation for decentralized
problem solving. However, even though several metrics have been 
proposed to quantify it, none of them is easily integrated with common solvers.
Constraint programming is a fundamental paradigm used to approach various families of problems.
We introduce Utilitarian Distributed Constraint Satisfaction Problems (UDisCSP) where the utility
of each state is estimated as the difference between the
the expected rewards for agreements on assignments for shared variables, 
and the expected cost of privacy loss. 

Therefore, a traditional DisCSP with privacy 
requirements is viewed as a planning problem. 
The actions available to agents are: communication and local inference.
Common decentralized 
solvers are evaluated here from the point of view of their interpretation 
as greedy planners. 
Further, we investigate some simple extensions where 
these solvers start taking into account the utility function.  
In these extensions we assume that the planning problem is further 
restricting the set of communication actions to only the communication primitives present
in the corresponding solver protocols.
The solvers obtained for the new type of problems propose the action (communication/inference)
to be performed in each situation, defining thereby the policy.
\end{abstract}

\IEEEpeerreviewmaketitle

\section{Introduction}

In Distributed Constraint Satisfaction Problems (DisCSP), agents have
to find values to a set of shared variables while respecting given
constraints (frequently assumed to have unspecified privacy implications). 
To find such assignments, agents exchange messages until a solution is 
found or until some agent detects that there is no solution to the 
problem. Thus, commonly agents reveal information during the solution 
search process, causing privacy to be a major concern in 
DisCSPs~\cite{yokoo1998distributed}.

The artificial intelligence assumption is that utility-based agents are able to associate
each state with a utility value~\cite{russell2010}. As such each action is associated with the difference between initial and final utilities. 
If users are concerned about their privacy, then such a user can 
associate a utility value with the privacy of each piece of information 
in the definition of their local problem. Since the users are interested in 
solving the problem, they must be also able to quantify the utility
each of them draws from finding the solution. Here we approach the
problem by assuming that privacy has a utility that can be aggregated
with the utility value of solving the problem. We evaluate how much
privacy is lost by the agents during the problem solving process, by
the total utility of each information that was revealed. The
availability of a value from the domain of a variable of the DisCSP in
the presence of the constraints of an agent, is the kind of
information that the agents want to keep private. For example, proposing
an assignment with that value has a cost quantifying the desire of the
agent to maintain its feasibility private. In traditional algorithms,
agents participate in the search process until an agreement is
found. We investigate the case where, being utility-driven, an agent
may stop its participation if the utility of the privacy expected to
be lost overcomes the reward for finding a solution of the
problem. Simple extensions to basic algorithms are investigated to
exploit the utilitarian model of privacy.

We then evaluate and compare synchronous and asynchronous algorithms
according to how well they preserve privacy. To do so, we generate
distributed meeting scheduling problems, as described
in~\cite{maheswaran2004taking,gershman2008scheduling}. In these
problems, all agents own one variable, corresponding to the meeting to
schedule, and the domain is the same for all variables. The
constraints consist in a global constraint that requires all the
variables to be equal, and also a unary constraint for each agent.

In the next section we discuss previous research concerning privacy
for DisCSPs. Further we formally define the concepts involved in
UDisCSPs. In Section~\ref{Algorithms} we introduce some extensions to
common DisCSP solvers that exploit the utilitarian model for
privacy. After a discussion on theoretical implications, we
present our experimental results in Section~\ref{Experimentations}.
We conclude in Section~\ref{Conclusions}.

\section{Background}	
\subsection{Backtracking Algorithms}
\subsubsection{Synchronous Backtracking}
The baseline algorithm for DisCSPs is the Synchronous Backtracking (\SyncBT{}), 
as presented
in~\cite{yokoo1992distributed}. \SyncBT{} is a simple distribution of
the standard backtracking algorithm.  The agents start by determining
a hierarchy between them. The higher priority agent then sends a
satisfying assignment of its variable to the next agent with an \okm{}
message. The recipient adds to it an
instantiation of its own variable while respecting its constraints,
and continues likewise. If an agent is unable to find an instantiation
compatible with the current partial assignment it has received, the
agent sends a \nogoodm{} message to the previous agent in the
hierarchy.  The process repeats until a complete solution is built, or
until the whole search space is explored.  The main efficiency concern
is that, since the messages are being sent sequentially, it does not
take advantage of possible parallelism.

\subsubsection{Asynchronous Backtracking}	
Asynchronous Backtracking (\ABT{})~\cite{yokoo1992distributed}, 
allows agents to run concurrently.
Each agent finds an assignment of its variable and communicates it to
the others agents,
having constraints involving this variable.
Agents then wait for incoming
messages. They receive an \okm{} message containing an assignment from
a related higher priority agent, at the beginning of the resolution
and also each time such an agent changes its assignment to avoid
constraint violation. 

An agent eventually receives values proposed by the agents it is
connected to by incoming links.  These values form a context
called \agentview{}. When an agent receives an \okm{} message, it
integrates the received assignment into its \agentview{} and checks
whether its own solution is consistent with it. If it is not the case,
the agent's assignment is changed.  The negation of a subset of
an \agentview{} preventing an agent from finding an assignment that
does not violate any of its constraints, is called a nogood.  If an
agent infers a nogood from its constraints and its \agentview{}, the
assignment of the lowest priority agent involved in the nogood has to
be changed. A \nogoodm{} message communicates to that agent the nogood,
which is treated by its recipient as a new constraint and can cause it
to change its assignment and generate corresponding \okm{}, \addlinkm() or
\nogoodm{} messages.

\subsection{Privacy}
In air traffic control~\cite{international2005worldwide}, each airport
has to allocate each takeoff and landing slots to the different
flights.  Even if airlines need combinations of slots to operate
sequences of flights, slots are currently allocated individually. Such
coordinated decisions are often impossible because of the need to keep
constraints private~\cite{faltings2008privacy}.
Thus, privacy has been an important aspect for DisCSP solving algorithms.
Recently, privacy preserving algorithms have also been developed for solving 
distributed constraint optimization problems \cite{grinshpoun2014privacy} \cite{tassa2015max}.

In existing works, there are two main approaches to enforce
privacy. The first one uses cryptographic techniques. The main problem
of these methods is that cryptographic protocols can be much slower,
which often makes them impractical~\cite{hirt2000efficient}.  The
second approach is based on using different search strategies to
minimize privacy loss, as defined by certain privacy metrics. In this
section we exemplify methods using these approaches.

\subsubsection{Sample Cryptographic Technique}	
As an example of cryptographic technique, the approach described
in~\cite{yokoo2002secure}, achieves a high level of privacy using
encryption, giving more importance to privacy than to the efficiency
of the resolution. It consists of using randomizable public key
encryption scheme. In this algorithm, three servers (value selector,
search controller and decryptor) receive encrypted information from
agents and cooperate to find an encrypted solution. Relevant parts of
the solution are then sent to each agent. This method guarantees that
no information is leaked to other agents. It also guarantees that,
thanks to the renaming of values by permutation, servers cannot know
the actual values they are dealing with. 
We now investigate methods that do not use cryptography. 

\subsubsection{Distributed Private Constraint Satisfaction Problems}	
A framework called Distributed Private Constraint Satisfaction
Problems (DisPrivCSPs), is introduced
in~\cite{freuder2001privacy,silaghi2002comparison},
modeling the privacy loss for individual revelations. It also models the
effect of the privacy loss by assuming that agents may abandon when
the incremental privacy loss overcomes the expected gains from finding
a solution.
Each agent pays a cost if the feasibility of some of its
tuple is determined by other agents. 
The reward for solving the problem is given
as a constant. 
Those concepts were so far used for evaluating qualitatively existing
algorithms, but were not integrated as heuristics in the search
process.   Privacy
and the cost/utility usual optimization criteria of Distributed Constraint
Optimization Problems
are merged in~\cite{doshi2008distributed} into a unique criterion.

\subsubsection{Valuations of Possible States}
The Valuations of Possible States (VPS)
framework~\cite{maheswaran2005valuations,maheswaran2006privacy,greenstadt2006analysis}
measures privacy loss by the extent to which the possible states of
other agents are reduced~\cite{freuder02}.
Privacy is interpreted as a valuation on the other agents'
estimates about the possible states that one lives in. 
During the search process, agents propose their values in an order of
decreasing preference. At the end of the search process, the
difference between the presupposed order of preferences and the real
one observed during search determines the privacy loss: the
greater the difference, the more privacy has been lost.

\subsubsection{Partially Known Constraints}	
The Partially Known Constraints (PKC) model~\cite{brito2009distributed},
uses entropy, as defined in information theory, to quantify privacy
and privacy loss. 
In this method, two variables $x_1$ and $x_2$ owned by two
different agents may share a constraint. However, not all the forbidden
couples ($x_1,x_2$) are known by both agents. Each agents only knows
a subset of the constraints. 
During the search process, assignment privacy is leaked
through \okm{} and \nogoodm{} messages, like in standard
algorithms. This problem is solved by not sending the value that is
assigned to a variable in a \okm{} message, but the set of values compatible
with this assignment. For \nogoodm{} messages, 
rather than sending the actual
assignments,
an identifier is used to specify the state of the resolution and is used to
check if some assignments are obsolete or not.

\section{Concepts}
The Distributed Constraint Satisfaction Problem (DisCSP) is the
formalism commonly used to model constraint problems distributed
between several agents.
It is represented by a quadruplet $\langle{A,V,D,C}\rangle$ where$:$

\begin{itemize}
	\item $A$: a set of agents.
	\item $V$: a set of variables, each one being owned by a distinct agent.
	\item $D$: a set of domains, each of them defining available values for the corresponding variable.
	\item $C$: 
	a set of constraints, each constraint being a relation imposed between two
	variables (i.e, $x_{1}=x_{2}$).
	
\end{itemize}

An agent that reveals an assignment to another agent, incurs a cost. 
Once the information is revealed, we consider that it becomes public, meaning that revealing it to yet another agent will not degrade its privacy.

\begin{example} \label{ex:1}
	Suppose a meeting scheduling problem between a professor and two 
	students. They all consider to agree on a time to meet on a given day,
	to choose between $8am$, $10am$ and $2pm$. For simplicity, in the next sections, we will refer to these possible values by their identifier: 1, 2, and 3. 
	The \ProfessorP{} is unavailable
	at $2pm$, \StudentAlice{} is unavailable at $10am$, and \StudentBob{} 
	is unavailable at $8am$.
	
	There can exist various reasons for privacy. For example,
	\StudentAlice{} does not want to reveal the fact that
	he is busy at $10am$ (because he secretly took a second job at that time). 
	The value \StudentAlice{} associates with not revealing the $10am$ unavailability is the salary from the second job ($\$2.000$). 
	The utility of finding an agreement for each student is the stipend for their studies ($\$5.000$).
	This is an example of privacy for absent values or constraint tuples.
	
	Further \StudentBob{} had recently boasted to \StudentAlice{} that at
	$8am$ he interviews for a job, and he would rather
	pay $\$1.000$ than to reveal that he is not.
	This is an example of privacy for feasible values of constraint tuples.
	
	Thus, corresponding agents associate a cost of $1$ to the revelation of their availability at 
	$8am$, equals to $2$ for the one at $10am$, and equals to $4$ for the one at $2pm$. 
	The reward from finding a solution is $5$.
	
\end{example}

\paragraph{DisCSP} The DisCSP framework models this problem with:	
\begin{itemize}
	\item
	$A=\{\Professor, \Alice, \Bob\}$
	\item
	$V=\{x_1,x_2,x_3\}$
	\item
	$D=$ $\{ \{1, 2, 3\}, \{1, 2, 3\}$, $\{1, 2, 3\} \}$
	\item
	$C= { \{x_1=x_2=x_3}, x_1\not=3, x_2\not=2, x_3\not=1  \}$
\end{itemize}
As it can be observed, DisCSPs cannot model the details regarding privacy
considerations. Now we will show how existing extensions model the
remaining details.

With DisPrivCSPs the additional parameters are $P$, to
specify the privacy coefficient of each value, and $R$, to specify
the rewards of each coefficient.
\begin{itemize}
	\item $P=\{P_\Professor{},P_\Alice{},P_\Bob{} \}=\{ (1,2,4),(1,2,4),(1,2,4)\}  $
	\item $R=\{R_\Professor{},R_\Alice{},R_\Bob{} \}=\{(5),(5),(5)\}  $
\end{itemize}

As we see, this framework successfully models all the information
described in the initial problem and also measures the privacy loss
for each agent. However, it was not yet investigated what is the
impact of the interruptions when privacy loss exceeds the revenue
threshold, or how agents could use these information to modify their
behavior during the search process to preserve more privacy.

\paragraph{VPS} For this problem, with 
the VPS framework, the 3 participants have to suppose an order of
preference between all different possible values for each other
agent. As agents initially do not know anything about others agents
but the variable they share a constraint with, they have to suppose an
equal distribution of all possible values for all other agents, meaning
that they do not expect the feasibility of any value to be less secret, and so proposed
first.  In this direction one needs to extend VPS to be able to also
model the kind of privacy introduced in this example.

\paragraph{PKC}
With PKC, the individual unavailabilities are only known by the corresponding participant. 
Only the junction of information known by the two agents over a given constraint can reconstruct the whole problem.
\begin{itemize}
	\item
	$A=\{\Professor, \Alice, \Bob\}$
	\item
	$V=\{x_1,x_2,x_3\}$
	\item
	$D= \{ \{1, 2, 3\}, \{1, 2, 3\}, \{1, 2, 3\} \}$
	\item 
	$C= \{ \{x_1=x_2=x_3, x_1\not=3\},$ \\
	\hspace*{0.92cm}$\{x_1=x_2=x_3, x_2\not=2  \},$ \\ 
	\hspace*{0.92cm}$\{x_1=x_2=x_3, x_3\not=1\} \}$
\end{itemize}

Extensions of PKC can also be proposed to model our example by adding
extra parameters for specifying the quantitative information about
privacy, as shown below.  Next we introduce a framework that can both
specify the quantitative input details, and can help agents in their
search process.

\paragraph{UDisCSP}
While some previously described frameworks do
model the details of our example, it has until now been an open
question as to how they can be dynamically used by algorithms in the
solution search process.

We propose to recast a DisCSP as a planning problem. It can be noticed
that the rewards and costs in our problem bear similarities with the
utilities and rewards commonly manipulated by planning algorithms  \cite{koenig2002heuristic}. As
such, we propose to define a framework which, while potentially being
equivalent in expressing power to existing DisCSP extensions, would
nevertheless explicitly specify the elements of the corresponding
family of planning problems.

We introduce the Utilitarian Distributed Constraint Satisfaction
Problem (UDisCSP). Unlike previous DisCSP frameworks, besides results,
we are also interested in the solution process. A policy is a function that associates
each state of an agent with an action that it should perform~\cite{russell2010}.

\theoremstyle{definition}

We define an {\em agreement} as a set of assignments for all the
variables with values from their domain, such that all the constraints
are satisfied. 

\begin{definition}
	A UDisCSP is formally defined as a tuple $\langle A,V,D,C,U,R \rangle$ where:
	\begin{itemize}
		\item
		$A=\langle A_1,...,A_n\rangle$ is a vector of $n$ agents
		\item
		$V=\langle x_1,...,x_n \rangle$ is a vector of $n$ variables.  
		Each agent $A_i$ controls 
		the variable $x_i$. 
		\item
		$D=\langle D_1,...,D_n\rangle$
		where $D_i$ is the domain for the variable $x_i$, known only to $A_i$, and a subset of $\{1,...,d\}$. 
		\item
		$C=\{c_1,...,c_m\}$ is a set of interagent constraints.
		\item
		$U=\{u_{1,1},...,u_{n,d}\}$ is a matrix of costs where
		$u_{i,j}$ is the cost of agent $A_i$ for revealing whether
		$j\in{}D_{i}$. 
		\item
		$R=\langle r_1,...,r_n\rangle$ is a vector of rewards,
		where $r_i$ is the reward agent $A_i$ receives if an agreement is found. 
	\end{itemize}
	The
	{\em state} of agent $A_i$ includes the subset of $D_i$ that it has
	revealed, as well as the achievement of an agreement. 
	The problem is to
	define a set of communication actions and a
	policy for each agent such that their utility is maximized.
	
\end{definition}

Note that the solution of a UDisCSP does not necessarily include
an agreement. In principle the set of available actions for agents
consist in any communication operator, as well as any local inference
computation.

\begin{example}
	The DisCSP in the Example~\ref{ex:1} is extended to UDisCSP 
	by specifying 
	the 
	additional parameters $U, R$: \\
	$U=\{u_{1,1}=1,u_{1,2}=2,u_{1,3}=4, \\
	\hspace*{0.925cm} u_{2,1}=1,u_{2,2}=2,u_{2,3}=4, \\
	\hspace*{0.925cm} u_{3,1}=1,u_{3,2}=2,u_{3,3}=4\}$. \\
	$R=\langle 5,5,5 \rangle$. \\	
\end{example}

The participants are utility-based agents~\cite{russell2010}
and try to reach the optimal state.

\section{Algorithms}\label{Algorithms}
\LinesNumbered
Now we discuss how the basic \ABT{} and \SyncBT{} algorithms are adjusted 
to UDisCSPs. 
The state of an agent includes the \agentview{}.
After each state change, each agent computes the estimated 
utility of the state reached by each possible action, and selects randomly 
one of the actions leading to the a state with the maximum expected utility.

In our algorithms, an information used by agents in their estimation
of expected utilities is the risk of one of their assignments being
rejected.  This risk can be re-evaluated at any moment based on data
recorded during previous runs on problems of similar parameters (e.g,
problem density).

The learning can be online or offline. 
For offline learning one calculates the number of 
messages \okm{} and \nogoodm{} sent during previous executions, called $count$.
It also counts how many messages previously sent lead to the termination
of the algorithm, in variable $terminationCount$. It calculates the risk for a 
solution to not lead to the termination of the algorithm, called $futilityRisk$.
Alternatively one can update
the $count$, $terminationCount$ and  $unsolvedRisk$ dynamically whenever the corresponding events happen.

\begin{equation}
\risk = 1 - \frac{terminationCount}{count}\label{eq:rejected}
\end{equation}

When \okm{} messages are sent, the agent has the choice of which assignment
to propose. When a \nogoodm{} message is scheduled to be sent, agents also
have choices of how to express them. 
Before each \okm{} or \nogoodm{} message, the
agents check which available action leads to the highest expected utility. 
If the highest expected utility is lower than the current one, the agent announces failure. 
The result is used to decide the assignment, nogood, or failure to perform.

We called these modified algorithms \SyncBTU{} and \ABTU{}, respectively. 
The algorithms \SyncBT{}U and \ABT{}U are obtained by 
performing the above mentioned modifications, in the pseudocodes
of \SyncBT{}~\cite{yokoo1992distributed,zivan2003synchronous} and of \ABT{}~\cite{yokoo1998distributed}.

\SyncBT{}U is obtained by restricting the set of communication actions to the
standard communication acts of \SyncBT{}, namely \okm{} and \nogoodm{} messages.
The procedures of a solver like \SyncBT{} define a policy, since they uniquely identify
a set of actions (inferences and communications) to be performed in each state. A state of an agent in \SyncBT{} is defined by an agent-view
and a current assignment of the local variable.
The local inferences in \SyncBT{}U are obtained from the ones of \SyncBT{} by a simple extension
exploiting the utility information available. The criteria in this research was not to guarantee
an optimal policy but to use utility with a minimal change to the original behavior of \SyncBT{} when
reinterpreted as a policy.
In \SyncBT{}U, the state is extended to also contain
a history of revelations of one's values defining an accumulated privacy loss,
and a probability to reach an agreement with each action.
Similar modifications are done to \ABT{} to obtain \ABTU{}: the restricted set 
of communication of \ABTU{} is composed of \okm{}, \addlinkm{} and \nogoodm{}.
The state and local inferences of \ABTU{} are the same as \SyncBTU{}, while also containing the set of nogoods.

For \ABT{}U, there are three procedures of \ABT{} that have to be modified: 
{\tt checkagentview}, {\tt when~nogood}, and {\tt backtrack}.
The new procedure  {\tt checkagentview} is shown in Algorithm~\ref{alg:ABTUcav}
and is obtained by inserting Lines~\ref{ln:abtu9} to 10. 
They test the privacy loss and only continue
as usual if the expected loss is smaller than the expected reward.

For lack of space, we do not include here the modified versions of the other
two procedures of \ABT{}U, since they are obtained in the same way from the
procedures of \ABT{} in~\cite{yokoo1998distributed}, procedure {\tt when~nogood}, $7^{th}$ line, and procedure {\tt backtrack}, $7^{th}$ line.
For \SyncBT{}U, its procedures are obtained from the procedures
of \SyncBT{} in an identical way as for \ABT{}U
and \ABT{}. Since~\cite{yokoo1992distributed} does not provide
pseudocode for \SyncBT{}, we refer the modifications to the pseudocode
presented in~\cite{zivan2003synchronous}, function {\tt assignCPA}, before Line~7, and function {\tt backtrack}, before Line~6.

\begin{algorithm}[h]

		\KwIn{$D$, $agentView$, $futilityRisk$, $reward$}
		\KwOut{}
		\SetKwFunction{algo}{algo}\SetKwFunction{proc}{proc}
		\SetKw{KwWhen}{when}
		\SetKw{KwDo}{do}
		\KwWhen $agentView$ and $currentValue$ inconsistent \KwDo{} \\
		{	\setcounter{AlgoLine}{1}
			\If {no value in $D$ is consistent with $agentView$}
			{backtrack\;}
			\Else{
				select d $\in$ $D$ where $agentView$ and d are consistent\;
				$currentValue$ $=$ d \;
				\If {calculateCost($futilityRisk$, $D$, $1$) $>=$ reward \label{ln:abtu9}}
				{interruptSolving()\;\label{ln:abtu10}}
				\Else 
				{send(ok?,($x_i$;d)) to outgoing links \label{ln:abtu11} } 
			}
		}
		\setcounter{AlgoLine}{0}
	\caption{procedure checkAgentView in ABTU}\label{alg:ABTUcav}

\end{algorithm}

To calculate the estimated utility of pursuing an agreement 
(revealing an alternative assignment), 
the agent considers all different possible scenarios of
the subsets of values that might have to be revealed in the future
based on possible rejections received, together with their probability
(see Algorithm~\ref{alg:CalculateCost}).  The algorithm assumes as
parameters: 
\begin{itemize}
	\item
	the previously calculated \risk{} (see Equation~\ref{eq:rejected}), 
	\item
	the possible values $D$, and 
	\item
	the probability of having to select from $D$.
\end{itemize}
The algorithm then recursively calculates the utility of the next possible
states, and whether the revelation of the current value $v$ leads to the
termination of the algorithm, values stored in variables $\costTerminal{}$ and
$costNonTerminal$. The algorithm returns the estimated cost of privacy
loss for the future possible states currently,
called \estimatedCost{}.

\newcounter{algoline}
\newcommand\Numberline{\refstepcounter{algoline}\nlset{\thealgoline}}
\LinesNumberedHidden
\begin{algorithm}[h]
	\KwIn{\risk{}, $D$, $\probabilityD$}
	\KwOut{$estimatedCost$}
	\Numberline \If {only one value is left in the domain}
	{\Numberline \Return (marginalCost(value) * $\probabilityD$)\;}
	\Else
	{\Numberline $v = first(D)$\;
	\Numberline \costTerminal = calculateCost
	\\ ~~~~ (1-\risk{},
	\{$v$\},
	$\probabilityD$)\;
	\Numberline \costNonTerminal = calculateCost
	\\~~~~ (\risk{},
	$D\setminus\{v\}$,
	\\~~~~ 
	\risk{} * $\probabilityD$)\;
	\Numberline estimatedCost = $\costTerminal + \costNonTerminal $\;
	\Numberline \Return estimatedCost\; 
	\caption{CalculateCost}\label{alg:CalculateCost} 
    } 
\end{algorithm}

\begin{example}
	
	Continuing with Example~\ref{ex:1}, at the beginning of the computation
	agent \Professor{} has to decide for a first action to perform.
	We suppose the $\risk{}$ learned from previous solvings is $0.5$.
	To decide whether it should propose an available value or not, 
	it calculates the corresponding \estimatedCost{} by calling 
	Algorithm~\ref{alg:CalculateCost} with parameters: 
	the learned $\risk{} = 0.5$, 
	the set of possible messages ($D=\{1,2,3\}$) and 
	$\probabilityD=1$. 
	
	For each possible value, this algorithm recursively sums the cost for the 
	two scenarios corresponding to whether the action leads immediately to termination, or not.
	Given privacy costs, the availability of three possible subsets of $D$ may be revealed in this problem: $\{1\}, \{1,2\}$, and $\{1,2,3\}$. The \estimatedCost{} returned is the sum of the costs for all 
	possible sets, weighted by the probability of their feasibility being revealed if an agreement is pursued. 
	At the call, $\costTerminal=u_{1,1}*0.5$. In the next recursion for $costNonTerminal$, we get $\costTerminal=(u_{1,1}+u_{1,2})*0.25$.
	In the last recursion, the algorithm returns $(u_{1,1}+u_{1,2}+u_{1,3})*0.25$.
	The \estimatedCost{} obtained is 
	$u_{1,1}*0.5+(u_{1,1}+u_{1,2})*0.25+(u_{1,1}+u_{1,2}+u_{1,3})*0.25$.
	The expected utility ($reward + \estimatedCost{}=5-3=2$)  of pursuing a solution being positive, 
	the first value is proposed.
	
\end{example}	

	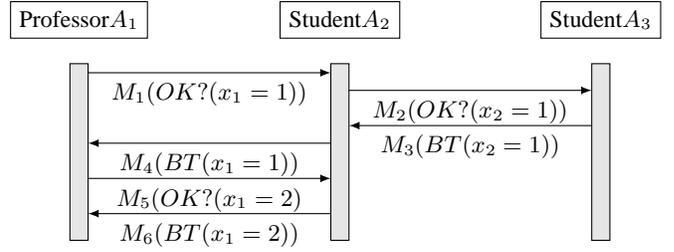
\begin{figure}[h]
		\begin{center}	
			\begin{tikzpicture}[every node/.style={font=\small,}]
			
			\node [matrix, very thin,column sep=0.75cm] (matrix) at (0,0) {
				& \node(0,0) (\Aa) {}; 						&&&&			 		\node(0,0) (\Ab) {}; 					&&&& 					\node(0,0) (\Ac) {}; & \\ 
				& \node(0,0) (\Aa 0) {}; 					&&&&		 			\node(0,0) (\Ab 0) {}; 					&&&&		 			\node(0,0) (\Ac 0) {}; & \\ 
				& \node(0,0) (\Aa 1) {}; 					&&&&					\node(0,0) (\Ab 1) {}; 					&&&& 					\node(0,0) (\Ac 1) {}; & \\ 
				& \node(0,0) (\Aa 2) {}; 			&&\node(0,0) (\M 1) {};&& 		\node(0,0) (\Ab 2) {}; 					&&&&					\node(0,0) (\Ac 2) {}; & \\ 
				& \node(0,0) (\Aa 3) {}; 					&&&&					\node(0,0) (\Ab 3) {};			&&\node(0,0) (\M 2) {};&&		\node(0,0) (\Ac 3) {}; & \\ 
				& \node(0,0) (\Aa 4) {}; 			 		&&&&					\node(0,0) (\Ab 4) {};					&&&&					\node(0,0) (\Ac 4) {}; & \\ 
				& \node(0,0) (\Aa 5) {}; 					&&&&					\node(0,0) (\Ab 5) {};			&&\node(0,0) (\M 3) {};&&		\node(0,0) (\Ac 5) {}; & \\ 
				& \node(0,0) (\Aa 6) {};			&&\node(0,0) (\M 4) {};&& 		\node(0,0) (\Ab 6) {}; 					&&&&					\node(0,0) (\Ac 6) {}; & \\ 
				& \node(0,0) (\Aa 7) {};					&&&&					\node(0,0) (\Ab 7) {}; 					&&&&					\node(0,0) (\Ac 7) {}; & \\ 
				& \node(0,0) (\Aa 8) {};			&&\node(0,0) (\M 5) {};&&		\node(0,0) (\Ab 8) {}; 					&&&&					\node(0,0) (\Ac 8) {}; & \\ 				
				& \node(0,0) (\Aa 9) {};					&&&&					\node(0,0) (\Ab 9) {}; 					&&&&					\node(0,0) (\Ac 9) {}; & \\ 		
				& \node(0,0) (\Aa 10) {};			&&\node(0,0) (\M 6) {};&&		\node(0,0) (\Ab 10) {}; 				&&&&					\node(0,0) (\Ac 10) {}; & \\ 
				& \node(0,0) (\Aa 11) {};					&&&&					\node(0,0) (\Ab 11) {}; 				&&&&					\node(0,0) (\Ac 11) {}; & \\ 		
			};

			\fill 
			(\Aa) node[draw,fill=white] {\Aaa}
			(\Ab) node[draw,fill=white] {\Aab}
			(\Ac) node[draw,fill=white] {\Aac};
			
			
			\filldraw[fill=gray!20]
			(\Aa 2.north west) rectangle (\Aa 11.south east)
			(\Ab 2.north west) rectangle (\Ab 11.south east)
			(\Ac 2.north west) rectangle (\Ac 11.south east);	
			
			\draw [-latex] (\Aa 2) -- (\Ab 2);
			\draw [-latex] (\Ab 3) -- (\Ac 3);
			\draw [-latex] (\Ac 5) -- (\Ab 5);
			\draw [-latex] (\Ab 6) -- (\Aa 6);
			\draw [-latex] (\Aa 8) -- (\Ab 8);
			\draw [-latex] (\Ab 10) --(\Aa 10);			
			\fill
			(\M 1) 
			node[below] {$M_{1} (OK?(x_{1}=1))$}
			(\M 2) 
			node[below] {$M_{2} (OK?(x_{2}=1))$}
			(\M 3) 
			node[below] {$M_{3} (BT(x_{2}=1))$}
			(\M 4) 
			node[below] {$M_{4} (BT(x_{1}=1))$}
			(\M 5) 
			node[below] {$M_{5} (OK?(x_{1}=2)$}
			(\M 6) 
			node[below] {$M_{6} (BT(x_{1}=2))$};	
			
			\end{tikzpicture}
			\caption{Interactions between agents during \SyncBT{}}\label{fig:SyncBT}
		\end{center}
	\end{figure}
	
	\begin{figure}[h]
		\begin{center} 
			\begin{tikzpicture}[every node/.style={font=\small,}]
			\node [matrix, very thin,column sep=0.65cm] (matrix) at (0,0) {
				& \node(0,0) (\Aa) {}; 						&&&&			 		\node(0,0) (\Ab) {}; 					&&&& 					\node(0,0) (\Ac) {}; & \\ 		
				& \node(0,0) (\Aa 0) {}; 					&&&&		 			\node(0,0) (\Ab 0) {}; 					&&&&		 			\node(0,0) (\Ac 0) {}; & \\ 
				& \node(0,0) (\Aa 1) {}; 					&&&&		 			\node(0,0) (\Ab 1) {}; 					&&&&		 			\node(0,0) (\Ac 1) {}; & \\ 
				& \node(0,0) (\Aa 2) {}; 			&&\node(0,0) (\M 1) {};&& 		\node(0,0) (\Ab 2) {}; 					&&&& 					\node(0,0) (\Ac 2) {}; & \\ 
				& \node(0,0) (\Aa 3) {}; 					&&&&		 			\node(0,0) (\Ab 3) {}; 			&&\node(0,0) (\M 2) {};&&		\node(0,0) (\Ac 3) {}; & \\ 
				& \node(0,0) (\Aa 4) {}; 					&&&&					\node(0,0) (\Ab 4) {};	 				&&&&					\node(0,0) (\Ac 4) {}; & \\ 
				& \node(0,0) (\Aa 5) {}; 					&&&&		 			\node(0,0) (\M 3) {}; 					&&&&		 			\node(0,0) (\Ac 5) {}; & \\ 
				& \node(0,0) (\Aa 6) {}; 					&&&&					\node(0,0) (\Ab 6) {};					&&&&					\node(0,0) (\Ac 6) {}; & \\ 
				& \node(0,0) (\Aa 7) {}; 					&&&&		 			\node(0,0) (\Ab 7) {}; 			&&\node(0,0) (\M 4) {};&&		\node(0,0) (\Ac 7) {}; & \\ 
				& \node(0,0) (\Aa 8) {}; 		&&\node(0,0) (\M 5) {};&&			\node(0,0) (\Ab 8) {};					&&&&					\node(0,0) (\Ac 8) {}; & \\ 
				& \node(0,0) (\Aa 9) {}; 					&&&&					\node(0,0) (\Ab 9) {}; 			&&\node(0,0) (\M 6) {};&&		\node(0,0) (\Ac 9) {}; & \\ 
				& \node(0,0) (\Aa 10) {}; 		&&\node(0,0) (\M 7) {};&&			\node(0,0) (\Ab 10) {};					&&&&					\node(0,0) (\Ac 10) {}; & \\ 
				& \node(0,0) (\Aa 11) {}; 				 	&&&&					\node(0,0) (\Ab 11) {}; 				&&&&					\node(0,0) (\Ac 11) {}; & \\ 
				& \node(0,0) (\Aa 12) {};					&&&&					\node(0,0) (\M 8) {}; 					&&&&					\node(0,0) (\Ac 12) {}; & \\ 
				& \node(0,0) (\Aa 13) {}; 					&&&&					\node(0,0) (\Ab 13) {}; 				&&&&		 			\node(0,0) (\Ac 13) {}; & \\ 
				& \node(0,0) (\Aa 14) {};		&&\node(0,0) (\M 9) {};&&			\node(0,0) (\Ab 14) {}; 				&&&&					\node(0,0) (\Ac 14) {}; & \\ 
				& \node(0,0) (\Aa 15) {}; 					&&&&					\node(0,0) (\Ab 15) {}; 				&&&&		 			\node(0,0) (\Ac 15) {}; & \\ 
			};
			
			\fill 
			(\Aa) node[draw,fill=white] {\Aaa}
			(\Ab) node[draw,fill=white] {\Aab}
			(\Ac) node[draw,fill=white] {\Aac};
			
			
			\filldraw[fill=gray!20]
			(\Aa 2.north west) rectangle (\Aa 15.south east)
			(\Ab 2.north west) rectangle (\Ab 15.south east)
			(\Ac 2.north west) rectangle (\Ac 15.south east);	
			
			\draw [-latex] (\Aa 2) -- (\Ab 2);
			\draw [-latex] (\Ab 3) -- (\Ac 3);
			\draw [-latex] (\Aa 5) -- (\Ac 5);
			\draw [-latex] (\Ac 7) -- (\Ab 7);
			\draw [-latex] (\Ab 8) -- (\Aa 8);
			\draw [-latex] (\Ab 9) -- (\Ac 9);
			\draw [-latex] (\Aa 10) -- (\Ab 10);
			\draw [-latex] (\Aa 12) -- (\Ac 12);	
			\draw [-latex] (\Ab 14) -- (\Aa 14);	
			
			\fill
			(\M 1) 
			node[below] {$M_{1} (OK?(x_{1}=1))$}
			(\M 2) 
			node[below] {$M_{2} (OK?(x_{2}=1))$}
			(\M 3) 
			node[below] {$M_{3} (OK?(x_{1}=1))$}
			(\M 4) 
			node[below] {$M_{4} (BT(x_{2}=1))$}
			(\M 5) 
			node[below] {$M_{5} (BT(x_{1}=1))$}
			(\M 6) 
			node[below] {$M_{6} (OK?(x_{2}=3))$}
			(\M 7) 
			node[below] {$M_{7} (OK?(x_{1}=2))$}
			(\M 8) 
			node[below] {$M_{8} (OK?(x_{1}=2))$}
			(\M 9) 
			node[below] {$M_{9} (BT(x_{1}=2))$};		
			
			\end{tikzpicture}
			\caption{Interactions between agents in \ABT{}}\label{fig:ABT}
		\end{center}
	\end{figure}
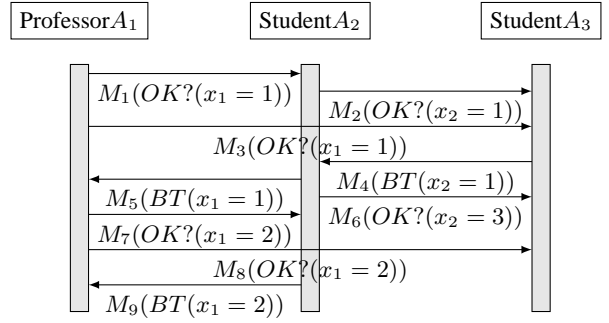

Next is an illustrative example
of other \ABTU{} operations.
\begin{example}	
	With the original \SyncBT{} and \ABT{}, possible obtained traces are depicted in 
	Figure~\ref{fig:SyncBT} and Figure~\ref{fig:ABT}, respectively.
	In Figure~\ref{fig:ABT}, we see that \StudentAlice{} proposes $x_2=1$ in 
	message $M_2$ and $x_2=3$ in message $M_6$.
	In this case, the privacy loss for \StudentAlice{} is $u_{2,1}+u_{2,3}=-1-4=-5$. 
	
	However, with \ABT{}U, we do not only use the actual utility of the
	next assignment to be revealed, but estimate privacy loss using 
	Algorithm~\ref{alg:CalculateCost}. After \StudentAlice{} has already sent $x_{2}=1$ with
	$M_2$, it considers sending $x_{2}=3$ with
	$M_6$. This decision making process is depicted in Figure~\ref{fig:inference}. 
	If the next value, $2pm$, is
	accepted, \StudentAlice{} will reach the final state while having
	revealed $x_2=1$ and $x_2=3$, for a total privacy cost of $u_{2,1} + u_{2,3} = -1 -4= -5$. If it is not,
	the unavailability of the last value $x_2=2$ will have to be revealed
	to continue the search process, leading to the revelation of all its
	assignments for a total cost of $-7$. Since both these scenarios have a probability
	of 50\% to occur, the \estimatedCost{} equals
	$((-5-7)/2)=-6$. The utility ($reward + \estimatedCost{}$) being equal 
	to $5-6=-1$, \StudentAlice{}
	has no interest in revealing $x_{2}=3$ and interrupts the solving. Its
	final privacy loss is only $u_{2,1}=-2$. The utility of the final state
	reached by \StudentAlice{} being $-2$ with \ABTU{}, and $-4$ with
	\ABT{}, \ABTU{} preserves more privacy than \ABT{} in this problem. 
\end{example}
	\begin{figure}[h]	
			\centering		
			\scalebox{0.9}{
		\tikzstyle{level 1}=[level distance=2cm, sibling distance=2cm]
		\tikzstyle{level 2}=[level distance=2cm, sibling distance=2cm]
		
		\tikzstyle{bag} = [text width=2em, text centered]
		\tikzstyle{end} = [circle, minimum width=3pt,fill, inner sep=0pt]
		
		\begin{tikzpicture}[grow=right]
		\node[bag] {8am}
		child {
			node[bag] {2pm}        
				child {	node[bag] {10am}    
					child{
						node[end, label=right:{$\sum\limits_{i\in \{1,2,3\} } u_{2,i}=-7$}] {}       		
							edge from parent         
							node[above] {}
							node[below]  {$1$}
						}
							edge from parent         
							node[above] {}
							node[below]  {$\frac{1}{2}$}
					}
						child {
							node[end, label=right:{$\sum\limits_{i\in \{1,3\} } u_{2,i}=-5$}] {}       		
							edge from parent         
							node[above] {}
							node[below]  {$\frac{1}{2}$}
						}					
			edge from parent         
			node[above] {}
			node[below]  {$\frac{1}{2}$}
		}
		child {
			node[end, label=right: {$\sum\limits_{i\in \{1\} } u_{2,i}=-1$}] {}
			edge from parent
			node[above] {}
			node[below]  {$\frac{1}{2}$}
		};
		\end{tikzpicture}	
}	
		\caption{Calculation of cost from \StudentAlice{} for all scenarios during \ABT{}U}\label{fig:inference}
	\end{figure}
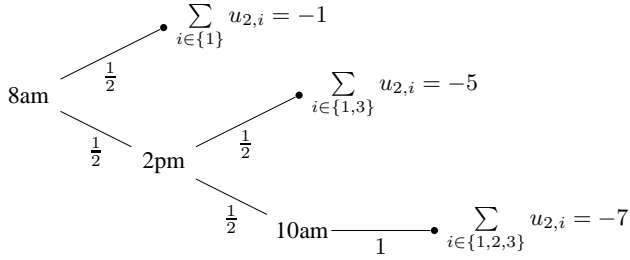

\paragraph{Theoretical Discussion}

The introduced UDisCSP framework can assume without significant loss
of generality that interagent constraints are public. This is due to
the fact that any problem with private interagent constraints (e.g.,
PKC), is equivalent with its dual representation where each constraint
becomes a variable~\cite{bacchus05}. However, note that privacy of
domains mentioned in~\cite{meseguer03} is not modeled by privacy of
constraints.

Moreover, the assumption that each agent owns a single variable is
also not restrictive.  Multiple variables in an agent can be
aggregated into a single variable by Cartesian product.

Nevertheless some algorithms can exploit these underlying structures
for efficiency, and this has been the subject of extensive 
research~\cite{fioretto2015multi}.

The UDisCSP mainly differs from DisCSP from the perspective of how solution is defined. It does not define solution as an agreement on a set of assignments but as a policy that could eventually reach such an agreement. As such, their comparison is not trivial, as one compares different aspects.
\begin{theorem}
	UDisCSPs planning and execution is more general than DisCSPs solving.
\end{theorem}
\begin{proof}
	A DisCSP can be modeled as a UDisCSPs with all privacy costs equal
	0. The obtained UDisCSPs would always reach an agreement, if
	possible. Therefore the goal of a UDisCSP would also coincide with the
	goal of the modeled DisCSP. This implies a tougher class of complexity.
\end{proof}

The space complexity required by \ABTU{} and \SyncBTU{} in each agent
is identical with the one of \ABT{} and \SyncBT{}, since the only
additional structures are:
\begin{itemize}
	\item
	the privacy costs associated
	with its values (constituting a constant factor increase for domain
	storage).
	\item
	the variables \risk{}, $terminationCount$, $count$ and $r_i$, which require
	a constant space.
\end{itemize}

\section{Experimental Results}\label{Experimentations}
We evaluate our framework and algorithms on randomly generated
instances of {\em distributed meeting scheduling problems
	(DMS)}. Previous work~\cite{wallace2005constraint} in distributed
constraint satisfaction problems has already addressed the question of
privacy in distributed meeting scheduling by considering the
information on whether an agent can attend a meeting to be
private. They evaluate the privacy loss brought by an action as the
difference between the cardinalities of the final set and of the initial
set of possible availabilities for a participant.
As different distributions of unary constraints can have an important impact on privacy leak, 
we generate two different types of random problems:	

\begin{enumerate}
	\item Uniform: Where the unary constraints are uniformly distributed between agents.
	\item Tail-constrained: Where the $n/2$ highest priority agents have a
	$3$ times lower probability to receive a unary constraint as compared
	to the $n/2$ lowest priority agents, even though the global density 
	remains the same.
\end{enumerate}

\begin{example}
	Suppose a problem where the two lowest priority agents have disjuncts sets 
	of availabilities, meaning that these agents can detect alone that the problem
	has no solution. \ABTU{} lets these agents exchange messages from the beginning of the search process and therefore interrupts it quickly.
	However, \SyncBTU{} prevents them from exchanging messages before all 
	higher priority agents have constructed a partial solution. Then, \SyncBTU{} 
	requires more messages exchange and therefore more privacy leak 
	than does \ABTU{}.
\end{example}	

The algorithm we use to  generate the problem is:

\begin{enumerate}
	\item We create the variables (one per {\em participant agents}).
	\item We initialize their domain (possible {\em times}).
	\item We add the global constraint {\em "all equals"}.
	\item Unary constraint are added to variables, to fit the density.
	\item For each value of each variable, we generate a
	revelation cost uniformly distributed between 0 and 9.
\end{enumerate}

The experiments are carried out on a computer under Windows 7, using a
1 core 2.16GHz CPU and 4 GByte of RAM. In Figure~\ref{fig:Plot}, we
show the total amount of privacy lost by all agents, averaged over 50
problems, function of the density of unary constraints. The problems
are parametrized as follows: 10 agents, 10 possible values, the utility
of a revelation is a random number between 0 and 9, the reward for
finding a solution to the problem is 20. Each set of experiments is 
an average esimation of 50 instances for the different algorithms 
(i.e, \SyncBT{}, \ABT{}, \SyncBTU{}, \ABTU{}).

	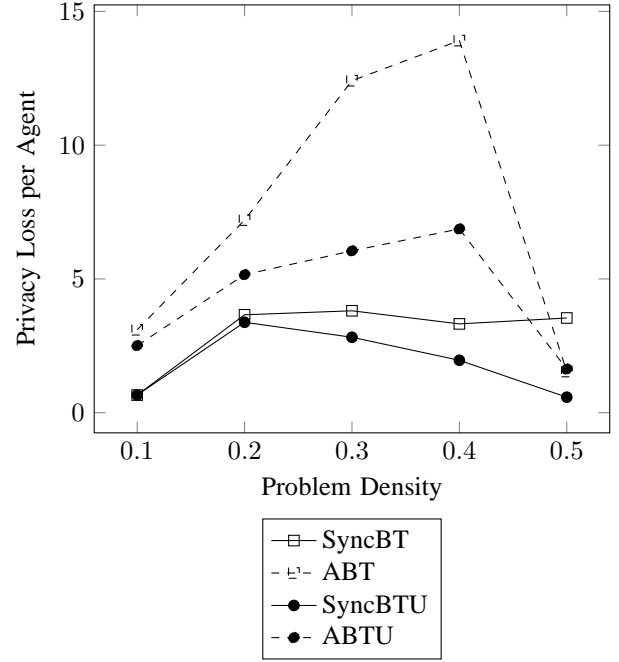
\begin{figure}[h]
		\centering
		\begin{tikzpicture}
		\begin{axis}[legend cell align=left,
		xlabel={Problem Density},
		ylabel={Privacy Loss per Agent}
		]
	
		\addplot [color=black, mark=square, solid] coordinates {
			(0.1,0.66) (0.2,3.66) (0.3,3.81)
			(0.4,3.32) (0.5,3.54) 
		};
		
		\addplot [color=black, mark=square, dashed]coordinates{
			(0.1,3.1) (0.2,7.2) (0.3,12.41)
			(0.4,13.91) (0.5,1.54) 
		};
		
		\addplot [color=black, mark= *, solid]coordinates{
			(0.1,0.66) (0.2,3.38) (0.3,2.82)
			(0.4,1.96) (0.5,0.58) 
		};
		
		\addplot [color=black, mark= *, dashed]coordinates{
			(0.1,2.51) (0.2,5.16) (0.3,6.05)
			(0.4,6.87) (0.5,1.63)
		};
		
		\pgfplotsset{every axis legend/.append style={
				at={(0.5,-0.2)},
				anchor=north},} 
		
		\legend{SyncBT, ABT, SyncBTU, ABTU}
		\end{axis}
		\end{tikzpicture}
		\caption{Evaluation of privacy loss 
			on different random problems}\label{fig:Plot}
	\end{figure}	
	\begin{figure}[h]
		\centering
		\begin{tikzpicture}
		\begin{axis}[legend cell align=left,
		xlabel={Problem Density},
		ylabel={Privacy Loss per Agent}
		]
	
		\addplot [color=black, mark= *, solid]coordinates {
			(0.1,0.69) (0.2,1.59) (0.3,2.19)
			(0.4,2.09) (0.5,1.64) 
		};
		
		\addplot [color=black, mark= *, dashed]coordinates{
			(0.1,2.96) (0.2,4.97) (0.3,6.47)
			(0.4,7.52) (0.5,8.61) 
		};
	
		\addplot [color=black, mark=square, solid]coordinates{
			(0.1,0.66) (0.2,3.38) (0.3,2.82)
			(0.4,1.96) (0.5,0.58) 
		};
		
		\addplot [color=black, mark=square, dashed]coordinates{
			(0.1,2.51) (0.2,5.16) (0.3,6.05)
			(0.4,6.87) (0.5,1.63)
		};
		
		\pgfplotsset{every axis legend/.append style={
				at={(0.5,-0.2)},
				anchor=north},} 
		
		\legend{Uniform - SyncBTU, Uniform - ABTU, Tail-Constrained - SyncBTU, Tail-Constrained - ABTU}
		\end{axis}
		\end{tikzpicture}
		\caption{Evaluation of privacy loss on instances with 
			different parameters.}\label{fig:Plot2}
	\end{figure}
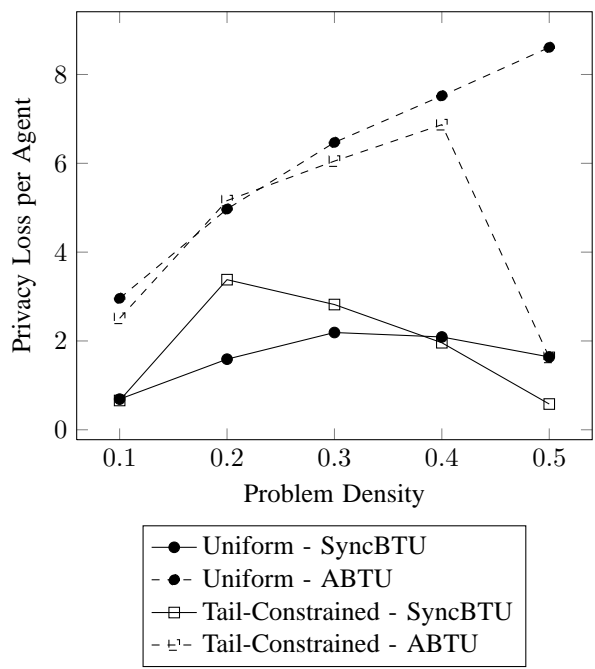	

	\begin{table}[]
		\caption{General comparison for algorithms along multiple metrics}\label{table}
		\centering
		\scalebox{1.0}{
		\begin{tabular}{llllllll}
			
            & SyncBT  & ABT      & SyncBTU & ABTU    \\
            PrivacyLoss & 2,5   & 9.0    & 1,8    & 5,3   \\
            Messages     & 2,8   & 531,3  & 2,3   & 150,6 \\
	        Solved       & 0,3    & 0,2     & 0,3    & 0,2    \\
            Interruption  & 0       & 0        & 0,3    & 0,7    \\
            cpuTime(ms)     & 257,7 & 1329,1 & 254,7 & 910,3
		\end{tabular}
	}
	\end{table}

\paragraph{Discussion on Experiments}
For each algorithm, we have measured in Table~\ref{table} the privacy 
loss, the number of messages exchanges, the number of problem solved, 
the number of solvings interrupted to preserved privacy and CPU time.
The Figure~\ref{fig:Plot} shows that synchronous algoriths are better than 
asynchronous at preserving privacy. Moreover, \SyncBTU{} and \ABTU{} are 
better than \SyncBT{} and \ABT{} at preserving privacy, respectively.
The Figure~\ref{fig:Plot2} shows the averaged estimation of privacy loss 
per agent, according to the density and the distribution of the 
unary constraints. We observe that for a density higher than $0.4$, 
privacy loss is lower for \textit{tail-constrained DMS} than it is for 
\textit{uniform DMS}, particularly with \ABTU{}.
We notice in Table~\ref{table} that interrupting the solvings to preserve privacy 
let agents not only reduce privacy loss by $39\%$ but also reduce the calculation 
times by $27\%$, reduce the number of messages exchanged by $71\%$ 
while still solving $98\%$ of the problems solved by standard algorithms, 
the interruptions happening mostly when the problems have no solution.

\section{Conclusions}\label{Conclusions}
While many frameworks have been developed recently for coping with
privacy in distributed problem solving, none of them is widely
used, likely due to the difficulty in modeling common problems. In this article we propose a framework called Utilitarian
Distributed Constraint Satisfaction Problem (UDisCSP). It models the
privacy loss for the revelation of an agent's constraints as a utility
function. We present algorithms that let agents estimate how much
privacy will be lost at the end of the solving process, using
information from previous experience in solving problems. This
estimation is used to modify the agent's behavior.
We then show how adapted synchronous and asynchronous
protocols (\SyncBTU{} and \ABTU{}) behave and compare them on different types of distributed meeting scheduling problems. The comparison shows that
\SyncBTU{} can maintain more privacy on random problems, as compared to both
\ABTU{} and original versions \ABT{} and \SyncBT{}. Some families
of problems with particular properties regarding privacy were also identified.

In future work, we want to investigate how much privacy
is leaked during the solving of different classes of problems. We also
plan to improve the way agents learn from previous experience, by
using not only the density of the corresponding problems, but also the
tightness, the number of variables or the number of interagent
constraints they are involved in.

\bibliographystyle{plain}	
\bibliography{wi}

\end{document}